\documentclass[letterpaper,conference]{IEEEtran}
\usepackage[left=0.75in,right=0.75in, top=0.75in,bottom=0.75in]{geometry}

\usepackage[bookmarks=false,colorlinks=true,urlcolor=blue,citecolor=blue,linkcolor=blue]{hyperref}
\usepackage{amsmath,amssymb,amsthm,cite}
\usepackage{graphicx}
\usepackage{color}
\usepackage{subcaption}
\usepackage{pbox}
\usepackage{tkz-euclide}
\usetkzobj{all}
\usepackage{tikz}
\usepackage{url}
\usepackage{todonotes}
\usepackage{bbm}
\usepackage{cleveref}
\usepackage{epstopdf}
\usepackage{algorithm}
\usepackage{algorithmic}
\usepackage{cleveref}
\pdfoutput=1

\newtheorem{lem}{Lemma}
\newtheorem{thm}{Theorem}
\newtheorem{defn}{Definition}
\newtheorem{coro}{Corollary}

\newtheorem{exple}{Example}

\newtheorem{rem}{Remark}

\newcommand{\E}[1]{\mathbb{E}\left[{#1}\right]}
\newcommand{\V}[1]{\mathrm{Var}\left[{#1}\right]}

\DeclareMathOperator*{\argmax}{arg\,max}

\graphicspath{{Figures/}}

\crefname{equation}{}{}
\Crefname{equation}{}{}
\crefname{thm}{theorem}{theorems}
\Crefname{thm}{Theorem}{Theorems}
\crefname{clm}{claim}{claims}
\Crefname{clm}{Claim}{Claims}
\Crefname{coro}{Corollary}{Corollaries}
\Crefname{lem}{Lemma}{Lemmas}
\Crefname{sec}{Section}{Sections}
\crefname{app}{appendix}{appendices}
\Crefname{app}{Appendix}{Appendices}
\crefname{prop}{proposition}{propositions}
\Crefname{prop}{Proposition}{Propositions}
\Crefname{propty}{Property}{Properties}
\crefname{figure}{fig.}{figures}
\Crefname{figure}{Fig.}{Figures}
\crefname{defn}{definition}{definitions}
\Crefname{defn}{Definition}{Definitions}
\crefname{fact}{fact}{facts}
\Crefname{fact}{Fact}{Facts}
\crefname{appendix}{appendix}{appendices}
\Crefname{appendix}{Appendix}{Appendices}
\crefname{algo}{algorithm}{algorithms}
\Crefname{algo}{Algorithm}{Algorithms}
\crefname{algorithm}{algorithm}{algorithms}
\Crefname{algorithm}{Algorithm}{Algorithms}
\crefname{conj}{conjecture}{conjectures}
\Crefname{conj}{Conjecture}{Conjectures}
\crefname{obs}{observation}{observations}
\Crefname{obs}{Observation}{Observations}

\begin{document}

\newcommand{\totalCaches}{m}
\newcommand{\nFork}{n}
\newcommand{\nPartialFork}{r}
\newcommand{\kJoin}{k}
\newcommand{\xm}{x_m}

\newcommand{\estimateProb}{\tilde{p}}
\def \OO {\mathrm{O}}
\def \oo {\mathrm{o}}


\title{Active Distribution Learning from Indirect Samples}

\author{\IEEEauthorblockN{Samarth Gupta}
\IEEEauthorblockA{Dept. 
of ECE \\
Carnegie Mellon University \\
Pittsburgh, PA 15213\\
Email: samarthg@andrew.cmu.edu}
\and
\IEEEauthorblockN{Gauri Joshi}
\IEEEauthorblockA{Dept. 
of ECE \\
Carnegie Mellon University \\
Pittsburgh, PA 15213\\
Email: gaurij@andrew.cmu.edu}
\and
\IEEEauthorblockN{Osman Ya\u{g}an}
\IEEEauthorblockA{Dept. of ECE\\
Carnegie Mellon University \\
Pittsburgh, PA 15213\\
Email: oyagan@ece.cmu.edu}
}


\maketitle


\begin{abstract}
This paper studies the problem of {\em learning} the probability distribution $P_X$ of a discrete random variable $X$ using indirect and sequential samples. At each time step, we choose one of the possible $K$ functions, $g_1, \ldots, g_K$ and observe the corresponding sample $g_i(X)$. The goal is to estimate the probability distribution of $X$ by using a minimum number of such sequential samples.
This problem has several real-world applications including inference under non-precise information and privacy-preserving statistical estimation. We establish necessary and sufficient conditions on the functions $g_1, \ldots, g_K$ under which asymptotically consistent estimation is possible. We also derive lower bounds on the estimation error as a function of total samples and show that it is {\em order-wise} achievable. Leveraging these results, we propose an iterative algorithm that i) chooses the function to observe at each step based on past observations;  and ii) combines the obtained samples to estimate $p_X$. The performance of this algorithm is investigated numerically under various scenarios, and shown to outperform baseline approaches.
\end{abstract}





\section{Introduction}
\label{sec:intro}

The modern world is rich with various types of data such as images, video, cloud job execution traces, social network data, and crowd-sourced survey data. These data can provide invaluable insights into the underlying random phenomenon which are generally not directly observable due to privacy concerns, or imprecise measurement mechanisms. For example, if we want to estimate the income distribution of a population, their salary data may not be public. However, it may be possible to estimate the income distribution using surveys about their spending on luxury goods, or whether their income is above or below some given thresholds. 

In this work we seek to design techniques to use indirect and correlated samples to estimate the probability distribution of a hidden random phenomenon. We consider a stylized model, shown in \Cref{fig:sys_model}, where a hidden variable $X$ can be sampled through functions $g_1(X),\ldots, g_K(X)$, referred to as \emph{arms}. Our objective is to accurately estimate the probability distribution of $X$ with the minimum number of samples; see Section \cref{sec:prob_formu} for a precise definition of the problem.

\subsection{Related Prior Work}
Learning the distribution of a random variable from its samples is a well-studied research problem \cite{valiant1984theory, kearns1994learnability,daskalakis2012learning} in information theory and theoretical computer science. Some works \cite{kamath2015learning, han2015minimax} are interested in finding the min-max or worst-case loss for various loss functions; e.g.,  L2-loss and Kullback-Liebler (KL) divergence. Some other works study the properties of distribution from samples observed \cite{acharya2017unified, acharya2014complexity,han2015adaptive, jiao2013universal}. Unlike the majority of the literature on distribution learning, here we assume that only functions $g_i(X)$ of the samples can be observed instead of direct samples of $X$. 


\begin{figure}[t]
\vspace{-4mm}
    \centering
    \includegraphics[width=0.43\textwidth]{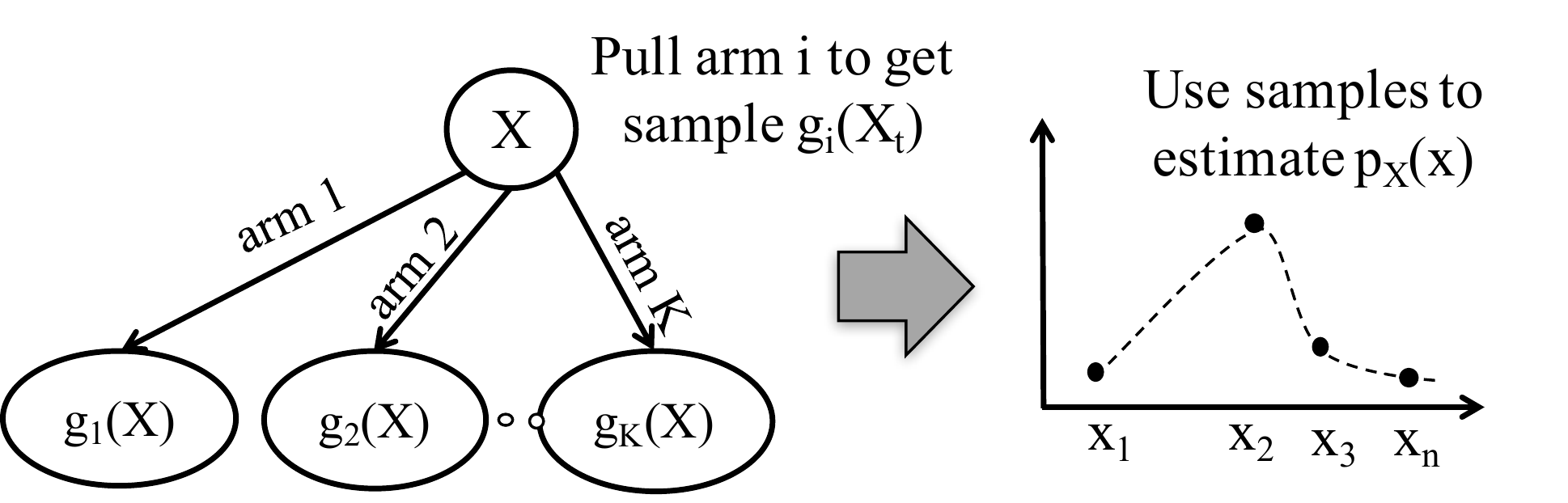}
    \caption{\sl At step $t$ we pull some arm $i$ and observe $g_i(X_t)$, where $X_t$ is an i.i.d.\ realization of the hidden variable $X$. Our objective is use the samples to estimate the distribution $p_X(x)$.}
    \label{fig:sys_model}
    \vspace{-4mm}
\end{figure}
Inferring a hidden random variable from indirect samples is also related to works in estimation theory \cite{kay1993fundamentals}, where the objective is to estimate a set of parameters $\theta$ using observations $y_1, \ldots, y_T$ that follow a model $p(Y|\theta)$. In our problem, the unknown distribution $P_X$ is analogous to the parameter $\theta$ while samples $g_i(X_t)$ correspond to the observations $y_1, \ldots, y_T$.  
%
%
A key difference between our model and typical parameter estimation problems is that we decide on the arm (say, arm $k$) to be pulled in each time slot $t$ to obtain the corresponding sample $g_k(X_t)$. Our problem formulation falls under the class of sequential design of experiments \cite{robbins1985some, chernoff1959sequential}. Such sequential/active learning frameworks have been considered for the purposes of hypothesis testing in \cite{naghshvar2013active, golovin2010near, javdani2014near}. The aspect of choosing arm in each time step is also closely related to the multi-armed bandit (MAB) sequential decision-making framework \cite{bubeck2012regret, auer2002finite,agrawal2012analysis}. In the classical MAB framework \cite{lai1985asymptotically}, each arm gives a reward according to some unknown distribution that is independent across arms, and the objective is to maximize the total reward for a given number of pulls, or to identify the arm that has the largest mean reward with as few pulls as possible\cite{jamieson2014best,audibert2010best,kaufmann2016complexity,gabillon2012best}. In contrast, the arms $g_1(X), \ldots, g_K(X)$ are correlated through the common hidden variable $X$ in our formulation. 
In most sequential experiment design and multi-armed bandit problems, the main strategy is to identify the single \lq\lq best" arm and then exploit it. What makes our formulation interesting is that there may not be a unique best arm for the purposes of learning the distribution of $X$. Instead, the optimal strategy will often involve a combination of arms to be pulled, with each arm being pulled a specific  number of times. It is also this aspect that makes our problem challenging since the optimal combination of arms to be pulled (to learn $P_X$) depends itself on the distribution $P_X$. 

\subsection{Main Contributions}
To the best of our knowledge, this is the first work to consider the problem of using sequential, indirect samples to learn the distribution of a hidden random variable. Our main contributions include i) deriving conditions on the functions $g_1(X), g_2(X), \ldots, g_K(X)$ needed for {\em asymptotically consistent} estimation of the hidden distribution; ii) deriving a lower bound on the estimation error and showing that it is {\em order-wise} achievable; and iii) proposing algorithms that sequentially decide which arm to pull and return an estimation of $P_X$ at each time step. Through simulations, our algorithms are also shown to outperform several baseline strategies in terms of error for a given number of pulls and the number of pulls needed to estimate $P_X$ within a given error. 





\section{Problem Formulation}
\label{sec:prob_formu}


Consider a discrete random variable $X$ that can take values from a finite alphabet $\{x_1, x_2, \ldots, x_n \}$ with an {\em unknown} probability distribution $P_X = [p_1, p_2, \ldots, p_n ]^\intercal$. Throughout this paper, we assume $p_i > 0$ for all $i$. Our objective is to estimate this probability distribution using a sequence of independent samples from $K$ functions $\{g_1, g_2, \ldots, g_K\}$, where each $g_i$ is a mapping from $\{x_1, x_2, \ldots, x_n\}$ to $\mathbb{R}$; throughout, we refer to these functions also as {\em arms}.
More precisely, with $\{X_t: t=1, 2, \ldots\}$ denoting a sequence of independent and identically distributed (i.i.d.) realizations of $X$, we can choose and observe only one of the $K$ possible outcomes $g_1(X_t), \ldots, g_K(X_t)$, at each step $t \in \mathcal{N}$. Broadly speaking, for a given set of functions $\{g_1, g_2, \ldots, g_K\}$, our goal is to derive an efficient 
algorithm i) to decide which function will be observed at each iteration step $t$, and ii) to come up with an estimate $\tilde{P}_X(t) = [\estimateProb_1(t), \estimateProb_2(t), \ldots, \estimateProb_n(t) ]^\intercal$ of the true probability distribution based on the observations until step $t$. Ultimately, we aim to minimize the mean-squared error of this estimation, formally defined below.





\begin{defn}[Estimation Error]
The error in estimating $P_X = [p_1, p_2, \ldots, p_n]^\intercal$ at step $t$ (i.e., after observing $t$ samples) is defined as
\begin{align}
\varepsilon(t) &= \E{\sum_{j=1}^{n} (\estimateProb_j(t)-p_j)^2}.
\end{align}
\label{defn:error_metric}
\end{defn}
Here, 
 $\estimateProb_i(t)$ denotes the estimation obtained after observing $t$ samples $g_{c_1}(X_1), g_{c_2}(X_2), \ldots, g_{c_t}(X_t)$, where $c_{\tau} \in \{1,\ldots,K\}$ is the arm pulled at step $\tau$.
We now give two examples to illustrate and clarify the problem formulation.  


\begin{exple}
\label{exple:Recoverable}
 \Cref{fig:Recovery} shows an example in which $X$ takes three possible values $\{x_1, x_2, x_3\}$, and there are three arms, $g_1, g_2$, and $g_3$. The values of $g_1, g_2$, and $g_3$ corresponding to $x_1, x_2,x_3$ are illustrated in \Cref{fig:Recovery}. In arm 1, output $z_{1,2}$ can come from either $x_2$ or $x_3$. This ambiguity exists in output $z_{2,2}$ (between $x_1$ and $x_3$) in $g_2$ and in output $z_{3,1}$ (between $x_1$ and $x_2$) in $g_3$. Inspite of these ambiguities, it is possible to estimate $p_1$, $p_2$ and $p_3$ as we will show in \Cref{sec:conditions}.
\end{exple}


\begin{exple}
\label{exple:NonRecoverable}
\Cref{fig:NoRecovery} illustrates an example with two arms, with each arm showing outputs corresponding to $\{x_1, x_2, x_3, x_4\}$. Arm $1$ has ambiguity coming from output of $x_2$ and $x_3$, whereas arm $2$ exhibits ambiguity in the output of $x_1, x_2$, and $x_3$. For this set of functions it is possible to estimate only $p_2 + p_3$ and nothing else can be known about $p_2$ and $p_3$, as we will prove in \Cref{sec:conditions}.
\end{exple}


We note that if a function $g_k$ is invertible, then every output sampled from $g_k$ will be uniquely matched to a single value (say, $x_j$) that $X$ can take without any ambiguity. In those cases, it would be optimal (in the sense of minimizing $\varepsilon(t)$ for each $t$) to pull $g_k$ at every step. We formally prove a more general version of this result in \Cref{thm:subset}.


\begin{figure}[t]
    \centering
    \includegraphics[width=0.43\textwidth]{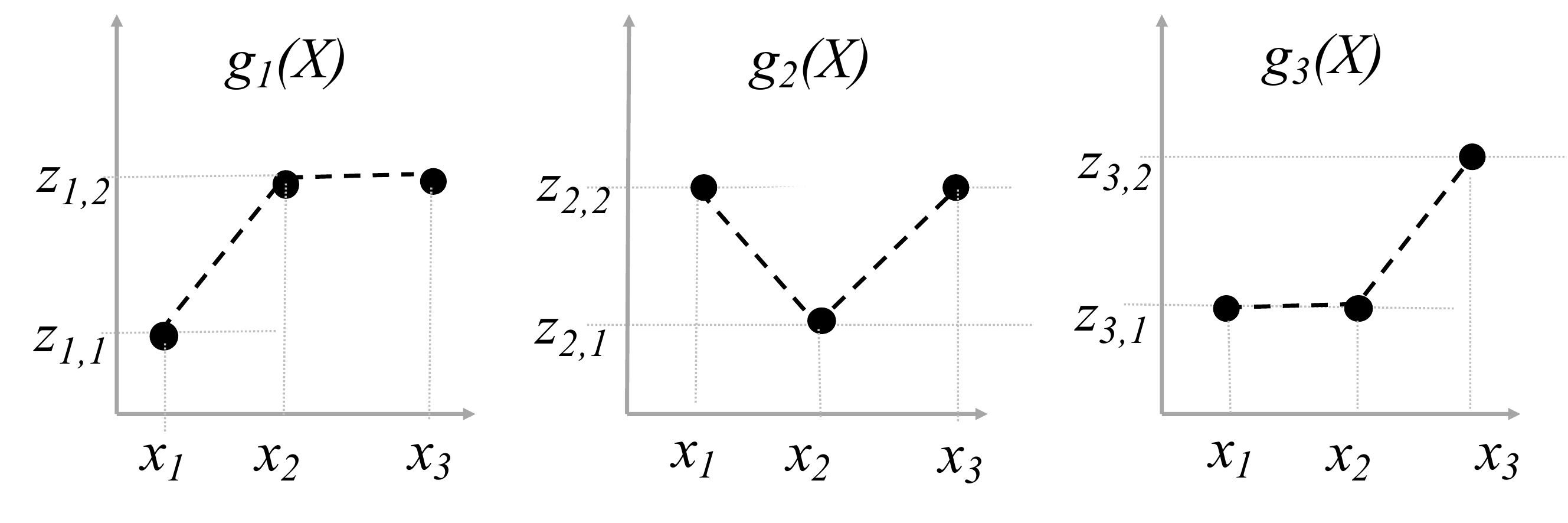}
    \vspace{-4mm}
\caption{\sl An example where it is possible to estimate $\{p_1, p_2, p_3\}$ asymptotically consistently (See \Cref{defn:consistent_estimation}) although no arm is invertible.} 
    \label{fig:Recovery}
\end{figure}

\begin{figure}[t]
    \centering
    \includegraphics[width=0.3\textwidth]{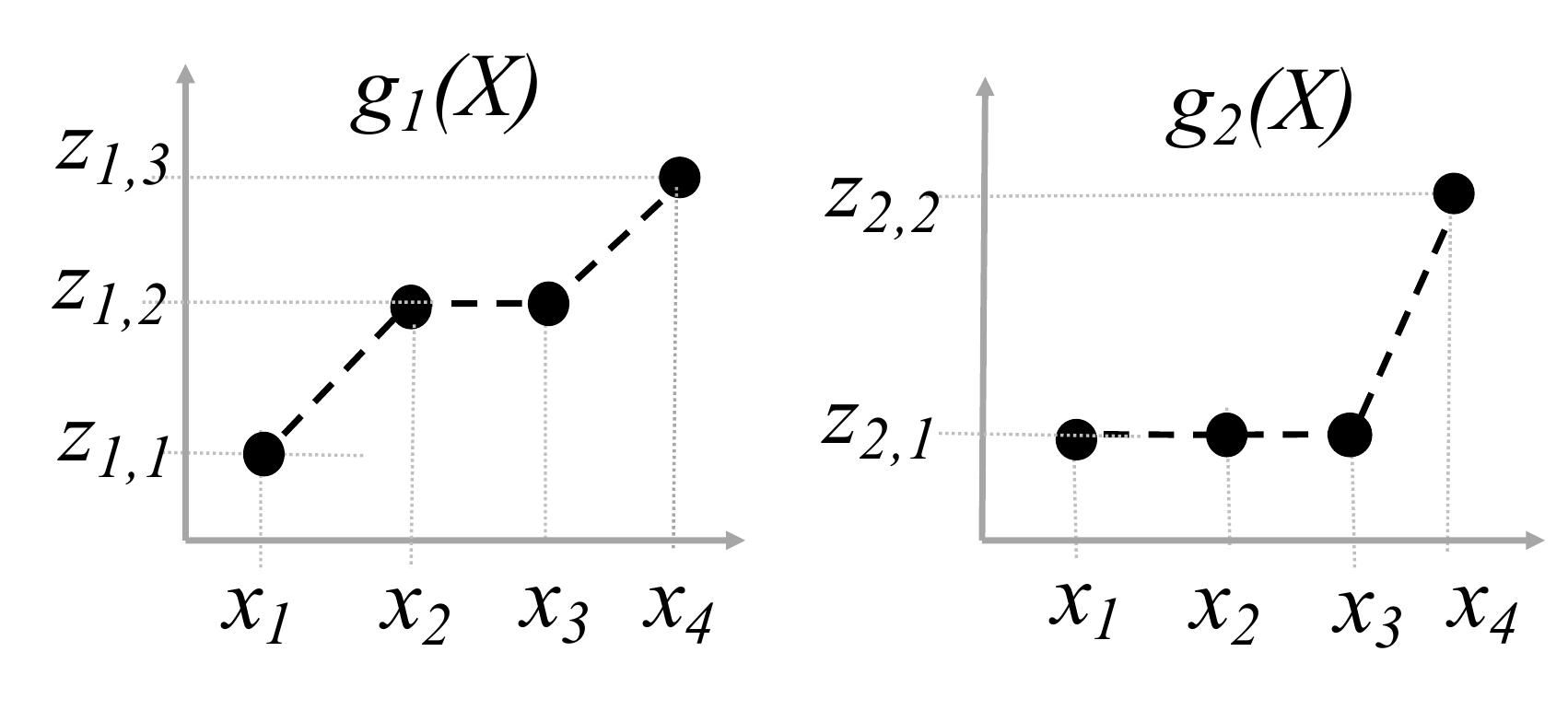}
        \vspace{-4mm}
\caption{\sl An example where it is not possible to get consistent estimation due to the ambiguity between $p_2$ and $p_3$.}    
    \label{fig:NoRecovery}
\end{figure}




\section{Structural Properties of the Functions $g_k(X)$}
\label{sec:resultsPrelem}

\subsection{Conditions for asymptotically consistent estimation}
\label{sec:conditions}




\begin{defn}[Asymptotically consistent estimation]
Given a random variable $X$ and arms $\{g_1, g_2, \ldots, g_K\}$, the estimated probability distribution $\{ \estimateProb_1(t), \dots,\estimateProb_n(t)\}$ is said to be asymptotically consistent if $\lim_{t \to \infty} \varepsilon(t) = 0$.
\label{defn:consistent_estimation}
\end{defn}


For each $k=1,\ldots, K$, let $\{z_{k,1}, z_{k,2}, \ldots, z_{k,m_k}\}$ denote the set of possible outcomes (i.e., the range of $g_k$) of the function $g_k$; evidently,  $m_k$ is the number of distinct outputs of $g_k$. The information about $g_k$ required to estimate $P_X$ can be captured in matrix  $A_k$ with $m_k$ rows and $n$ columns, where 
 \[
 A_k(i,j) = \left\{
 \begin{array}{ll}
    1,  &  \textrm{if $g_k(x_j) = z_{k,i}$}  \\
    0,  &   \textrm{otherwise},
 \end{array}
 \right.
 \]
 for each $i=1,\ldots, m_k$ and $j=1,\ldots, n$.
  Informally, $A_k(i,j) = 1$ if output $z_{k,i}$ could have been generated by $x_j$ in arm $k$. We refer $A_k$ as the \textit{Sample Generation Matrix} for arm $k$.
Let the matrix $A$ be given by $A = [A_1^\intercal, A_2^\intercal, \ldots, A_K^\intercal]^\intercal$; the size of $A$ is $m \times n$, where $m=m_1 + \ldots + m_K$. The corresponding matrices $A^{\textrm{Example-1}}$ and $A^{\textrm{Example-2}}$ for Examples \ref{exple:Recoverable} and \ref{exple:NonRecoverable}, respectively are shown below.
\[ 
A^{\textrm{Example-1}}  = \left[ \begin{array}{ccc}
1 & 0 & 0 \\
0 & 1 & 1 \\
1 & 0 & 1 \\
0 & 1 & 0\\
1 & 1 & 0\\
0 & 0 & 1
\end{array} \right],
 ~~~ A^{\textrm{Example-2}}= \left[ \begin{array}{cccc}
1 & 0 & 0 & 0\\
0 & 1 & 1& 0 \\
0 & 0 & 0 & 1\\
1 & 1 & 1 & 0 \\
0 & 0 & 0 & 1
\end{array} \right]
\]

\begin{thm}
\label{thm:achievable}
It is possible to achieve asymptotically consistent estimation if and only if $\textrm{rank}(A) = n$.
\end{thm}

\begin{proof}[Proof of \Cref{thm:achievable}]
Recall that $z_{k,i}$ represents the $i^{th}$ distinct output of arm $k$. Let $q_{k,i}$  denote the probability of observing $z_{k,i}$ each time arm $k$ is pulled. Consider the system of linear equations below relating these probabilities to the probability distribution of $X$:
\begin{equation}
q_{k,i} \triangleq \sum_{j = 1}^{ n} A_k(i,j)p_j, \qquad \begin{array}{c}
     k=1,\ldots, K  \\
     i=1,\ldots, m_k 
\end{array}  
\label{rankEqn}
\end{equation}   
These set of equations can be written as $$A P_X = Q,$$ with $Q$ denoting the vector $[{q}_{1,1}, \ldots, {q}_{1,m_1}, \ldots {q}_{K,m_K}]^\intercal$.

Suppose now that $A$ is full rank. In order to  construct an asymptotically consistent estimate of $P_X = [p_1, \ldots, p_n]^\intercal$, assume that arms are pulled in a round-robin manner. Thus, at step $t$ we will have $\frac{t}{K}$ samples from each arm. With $t_{k,i}$ denoting the number of times $z_{k,i}$ is observed in $t$ steps, we let $\hat{q}_{k,i}(t) = \frac{t_{k,i}}{{t}/{K}}$
be the estimate of $q_{k,i}$ at step $t$. 
By virtue of Strong Law of Large Numbers, we have $\hat{q}_{k,i}(t) \to q_{k,i}$ almost surely as $t$ goes to infinity, that is,  the estimates $\hat{q}_{k,i}(t)$ are asymptotically consistent. 
 Given that $A$ is full rank, the estimates $\hat{q}_{k,i}(t)$  can be used to obtain a unique solution of $P_X = [p_1, p_2, \ldots, p_n]^\intercal$ 
from the system of equations \Cref{rankEqn}. Given that $n$ is finite, this unique solution will constitute an asymptotically consistent estimation of $P_X$ as well.

Conversely, if $\textrm{rank}(A) < n$,  it is not possible to obtain a unique solution of the system of equations in \Cref{rankEqn}. This implies that even if consistent estimation of each $q_{k,i}$ is possible, it is not possible to achieve asymptotically consistent estimation of the probability distribution, $P_X$.
\end{proof}



Clearly, $\textrm{rank}(A^{\textrm{Example-1}})= n$ while $\textrm{rank}(A^{\textrm{Example-2}}) < n$. Thus,
asymptotically consistent estimation is possible for the set of functions in \Cref{exple:Recoverable} but not in \Cref{exple:NonRecoverable}.

\begin{rem}
  \Cref{thm:achievable} 
is not constrained to the Definitions \ref{defn:error_metric} and \ref{defn:consistent_estimation} of error and asymptotically consistent estimation, respectively. 
 In fact, the condition $\textrm{rank}(A) = n$ is   necessary and sufficient to have (the possibility of achieving) for any $\epsilon > 0$ and $\forall i$ that  $\mid \estimateProb_i(t) - p_i \mid < \epsilon$ for all $t$ sufficiently large.
\end{rem}

\subsection{Redundant functions/arms}


Recall the definition of sample generation matrix $A_k$ for each arm $k$ given in \Cref{sec:conditions}. 


\begin{defn}[Redundant Arm]
\label{def:subsetarm}
An arm $r$ is said to be a \emph{redundant} if there exists another arm $s$ such that the row space of $A_r$ is a strict subset of the row space of $A_s$.
\end{defn}

Informally, this means that all information produced by arm $r$ can be generated by arm $s$. For example, in \Cref{fig:NoRecovery} we see that arm 2 generates information about $p_1 + p_2 + p_3$, while arm 1 generates information about $p_1$ and $p_2 + p_3$ separately; also, both arms generate information about $p_4$ separately. Therefore, information produced by arm 2 can be generated by arm 1. This observation is made precise next.

\begin{thm}
\label{thm:subset}
If an arm $r$ is redundant, then it is suboptimal to pull arm $r$ at any step $t$ for the purpose of minimizing $\varepsilon(t)$.
\end{thm}
\begin{proof}[Proof of \Cref{thm:subset}]
Since arm $r$ is redundant, it implies that there exists an arm $s$ such that the row space of $A_r$ is a strict subset of the row space of $A_s$. Suppose we are given a set $\mathcal{Z}_s$ of the samples from arm $s$. For each observation $z_{s,i} \in \mathcal{Z}_s$, consider the set $\mathcal{X}_{s,i} = \{x: g_s(x) = z_{s,i}\}$. Since the row space of $A_r$ is a subset of row space of $A_s$, each  $x \in \mathcal{X}_{s,i}$ will be mapped to the same observation $g_r(x)$ in arm $r$. 
More formally, we have $g_{r}(x) = z_{r,i}$ $\forall{x \in \mathcal{X}_{s,i}}$. Repeating this for all $z_{s,i} \in \mathcal{Z}_s$, we can construct a new sample set $\mathcal{Z}_r$. For the underlying set of realizations $\mathcal{X}=\{X_1, X_2, \ldots\}$ that lead to the samples $\mathcal{Z}_s$, we can see that $\mathcal{Z}_r$ is the exact sample set that one would have obtained if arm $r$ was pulled each time instead of arm $s$. This shows that if arm $r$ is redundant, then each time the arm $s$ is pulled, we automatically know the sample that we would have obtained from arm $r$, thereby obviating the need to ever pull arm $r$. Thus, for the purposes of minimizing $\varepsilon(t)$, it is suboptimal to ever pull a redundant arm $r$; pulling the arm $s$ at each time step will be at least as good. 
 \end{proof}

 By \Cref{thm:subset}, if an \textit{invertible} arm exists, all other arms will be redundant. This leads to the following corollary.

\begin{coro}
\label{alwaysInvertible}
If there is an invertible arm, then the optimal action
(for the purpose of minimizing $\varepsilon(t)$)
is to pull the invertible arm at every step.   
\end{coro}

\begin{rem}
The proofs of \Cref{thm:subset} and \Cref{alwaysInvertible} are not specific to the error metric given in \Cref{defn:error_metric}. Thus, both results hold  true under other error metrics as well; e.g., L1 norm, KL divergence etc.
\end{rem}

\section{Bounds on the Estimation Error}
\label{sec:ResultsBounds}
\subsection{Lower bounds on the estimation error}
\label{sec:bounds}


We first derive a {\em crude} lower bound on the estimation error which does not depend on the functions $\{g_1, g_2, \ldots g_K\}.$ 

\begin{thm}[Crude Lower Bound]
\label{thm:universalBound}
Estimation error of any unbiased estimator for the problem in \Cref{sec:prob_formu} is lower bounded by $\sum_{j=1}^{n} \frac{p_j(1 - p_j)}{t}$.
\end{thm}
\begin{proof}[Proof of \Cref{thm:universalBound}]

From \Cref{alwaysInvertible}, we know that it is optimal to always pull the invertible arm if there exists one. It is also clear that the optimal error can only decrease when an additional arm is included in the set of possible arms we can choose. Thus, for the purpose of deriving a lower bound on the estimation error, we can assume the existence of an invertible arm which is pulled in all $t$ steps.


We define 
$\hat{p}_j(t) = \frac{t_{x_j}}{t},    
$
as the corresponding empirical estimator (which is also the  maximum likelihood estimator), where $t_{x_j}$ is the number of times the output corresponding to $x_j$ was observed (from the invertible arm) in $t$ steps. Under this scenario, 
the estimation error is given by
\begin{equation}
\varepsilon(t)=\sum_{j=1}^{ n}\V{\hat{p}_j(t)} = \sum_{j=1}^{n} \frac{p_j(1 - p_j)}{t},
\end{equation}
as $t_{x_j}$ is a Binomial random variable, which has variance $t p_j (1-p_j)$. This also gives the minimum possible variance for any unbiased estimator (given the samples from the invertible arm). Using this fact and \Cref{alwaysInvertible}, we establish \Cref{thm:universalBound}.

\end{proof}


\begin{rem}
The lower bound in \Cref{thm:universalBound} is achieved if an invertible arm exists.
\end{rem}

The final estimation error after a total of $t$ steps depends on the number of times each arm is pulled till step $t$. Due to the sequential nature of the problem, we have control over which arm is pulled at each time step and hence on the number of times each arm $k$ is pulled till step $t$, i.e., $t_k$. Next, we derive a lower bound on the error of any unbiased estimator given the number of times each arm is pulled.  


\begin{thm}[Lower bound on error for a given number of pulls]
Let $\mathbf{t} =  [t_1, t_2, \ldots, t_K]^\intercal$ be the number of times arms $\{g_1, \ldots, g_K\}$ are pulled, respectively. The estimation error of any unbiased estimator satisfies
\label{thm:CRLB}
\begin{align}
  \varepsilon(\mathbf{t}) \geq tr(I(\theta, \mathbf{t})^{-1}) + \sum_{i=1}^{n-1} \sum_{j=1}^{n-1} I(\theta, \mathbf{t})^{-1}(i,j), 
\end{align}

where $I(\theta, \mathbf{t})$ is the  $n-1 \times n-1$ Fisher-Information 
matrix with entries 
\begin{align}\nonumber
I_{i,j}(\theta,\mathbf{t})  &= \sum_{k=1}^{K}\sum_{\ell = 1}^{m_K} \frac{t_k   A_k(\ell,i)  A_k(\ell,j)  (1 - A_k(\ell,n))}{q_{k,\ell}} + \\
&~ \frac{t_k  (1-A_k(\ell,i))  (1 - A_k(\ell,j)  A_k(\ell,n)}{q_{k,\ell}}.
\label{eq:fisher_info_matrix}
\end{align}
\end{thm}
\begin{proof}
We use the Cramer-Rao bound \cite{rao1945information, cramer1946mathematical} that provides a lower bound on the covariance matrix of any unbiased estimator of an unknown deterministic parameter. Since $\sum_{i = 1}^{n} p_i = 1$ it suffices to estimate any $n-1$ of the parameters $\{p_1, p_2, \ldots, p_n \}$. Let these parameters ($\theta = \{\theta_1, \theta_2, \ldots, \theta_{n-1} \}$) be $\{p_1, p_2, \ldots, p_{n-1}\}$. Let $\mathcal{D}_t$ be the event that after $t$ steps, output $z_{k,i}$ from arm $k$ is observed $t_{k,i}$ times, for all $k \in [1,K]$, and $i \in [1, m_k]$.

We evaluate the log likelihood $L(\mathcal{D}_t;\theta)$ of observed data $\mathcal{D}_t$ with respect to $\theta$, 
We then compute the $n-1 \times n-1$ Fisher information matrix, $I(\theta, \mathbf{t})$, whose
 $(i,j)^{th}$ entry is given by $-\E{\frac{\partial^2}{\partial \theta_i \partial \theta_j} L(\mathcal{D}_t;\theta) | t_1, \ldots, t_K}$. For our problem, we obtain a closed form expression of $I_{i,j}(\theta, \mathbf{t})$ given in \cref{eq:fisher_info_matrix}.


The Cramer-Rao lower bound on covariance matrix of $\theta$ for any unbiased estimator is then given by $I(\theta, \mathbf{t})^{-1}$. Our objective is to minimize $\sum_{i = 1}^{n} \V{\estimateProb_i}$, which can be  bounded as 
\begin{align}
\sum_{j = 1}^{n} \V{\estimateProb_j} &= \sum_{j = 1}^{n-1} \V{\estimateProb_j} + \V{\estimateProb_n}, \\
&\geq tr(I(\theta, \mathbf{t})^{-1}) + \V{1-\sum_{j=1}^{n-1}\estimateProb_j}, \\
&= tr(I(\theta, \mathbf{t})^{-1}) +  \sum_{i = 1}^{n-1}\sum_{j=1}^{n-1} \textrm{Cov}(\estimateProb_i, \estimateProb_j), \label{bound} \\
&\geq tr(I(\theta, \mathbf{t})^{-1}) + \sum_{i=1}^{n-1} \sum_{j=1}^{n-1} I(\theta, \mathbf{t})^{-1}(i,j).
\end{align}
\end{proof}

Since the inverse Fisher information matrix, $I(\theta, \mathbf{t})^{-1}$, is also a lower bound on the covariance of any estimator that exhibits local asymptotic normality, therefore, when $t \rightarrow \infty$, the result in \Cref{thm:CRLB} also holds for any estimator which is asymptotically normal locally or exhibits asymptotical minimaxity. We now state the lower bound on estimation error for biased estimator with bias $b(\theta)$. 

\begin{thm}[Lower bound for any estimator with given bias]
Let terms $\mathbf{t}$ and $I(\theta, \mathbf{t})$ be defined as in the statement of \Cref{thm:CRLB}. The estimation error of any biased estimator with bias $b(\theta)$ satisfies 
\[
\epsilon(\mathbf{t}) \geq tr(\bar{I}(\theta, \mathbf{t})^{-1})  + \sum_{i = 1}^{n-1}\sum_{j = 1}^{n - 1} \bar{I}(\theta,\mathbf{t})^{-1}(i,j) + ||b(\theta)||_2^2,
\]
where $\bar{I}(\theta, \mathbf{t}) = \frac{\partial (\theta - b(\theta))}{\partial \theta} I(\theta, \mathbf{t})^{-1} \left(\frac{\partial (\theta - b(\theta))}{\partial \theta}\right)^{-1}$. 
\label{thm:CRLB_bias}
\end{thm}

\begin{proof}
The proof follows from the fact that Cramer-Rao lower bound on the covariance matrix of $\theta$ for any biased estimator with bias $b(\theta)$ is given by  $\bar{I}(\theta,\mathbf{t})^{-1}$, with $\bar{I}(\theta, \mathbf{t}) = \frac{\partial (\theta - b(\theta))}{\partial \theta} I(\theta, \mathbf{t})^{-1} \left(\frac{\partial (\theta - b(\theta))}{\partial \theta}\right)^{-1}$. We can evaluate the $\text{Var}(\theta)$ as in proof of \Cref{thm:CRLB}. Using this with the fact that $\text{MSE}(\theta) = \text{Var}(\theta) + ||b(\theta)||_2^2$ gives us \Cref{thm:CRLB_bias}.
\end{proof}



%

\subsection{Orderwise Achievability}
\label{achievability}

In \cref{sec:bounds}, we showed that $\epsilon(t) = \Omega \left(\frac{1}{t}\right)$. We now show that this lower bound is achievable if $rank(A) = n$ by analyzing the estimation error of \textsc{RRpull$+$PIest} algorithm. The \textsc{RRpull$+$PIest} pulls arms in a round-robin manner and uses the pseudo inverse of the matrix $A$ to produce estimate $\tilde{P}_X(t)$ at each time step $t$. Formal description of \textsc{RRpull$+$PIest} is given in \Cref{alg:NaiveAlgo}.  
\begin{algorithm}[t]
\hrule 
\vspace{0.1in}
\begin{algorithmic}[1]
\STATE \textbf{Input:} $\{x_1, x_2, \ldots, x_n\}$, Functions $\{g_1, g_2 \ldots g_K\}$ where $g_i: \{x_1, x_2, \ldots, x_n \} \rightarrow \mathbb{R}$. Total number of steps, $T$. 
\STATE \textbf{Initialize:} $t_k = 0$ $\forall k.$ $t_{k,i} = 0, \forall{i,k}.$ $\estimateProb_j(0) = \frac{1}{n}, \forall{j}$. 
\FOR{$t = 1:T$}
\STATE $c_t = \mod(t,K) + 1$
\STATE Pull arm $c_t$, observe output $y_t$
\STATE $t_k = t_k + 1$ 
\IF {$y_t = z_{k,i}$}
\STATE $t_{k,i} = t_{k,i} + 1$
\ENDIF
\STATE $\hat{q}_{k,i} = \frac{t_{k,i}}{t_k}$ $\forall{i,k}.$
\STATE Obtain estimates $\estimateProb_j(t)$ as  $\tilde{P}_X = A^+\hat{Q}$. 
\ENDFOR
\end{algorithmic}
\vspace{0.1in}
\hrule
\caption{\textsc{RRpull $+$ PIest}}
\label{alg:NaiveAlgo}
\end{algorithm}





\begin{thm}[Order-wise Achievability]
\label{thm:orderwise}
It is possible to achieve estimation error of $\OO\left(\frac{1}{t}\right)$ if $rank(A) = n$.
\end{thm}
\begin{proof}[Proof of \Cref{thm:orderwise}]
 In order to show achievability we consider the \textsc{RRpull + PIest} algorithm that pulls arm in a round-robin manner due to which each arm is pulled $\frac{t}{K}$ times in $t$ steps. 
 For each $k=1,\ldots, K$ and $i=1,\ldots, m_k$ let  $\hat{q}_{k,i}(t)=\frac{t_{k,i}}{{t}/{K}}$ 
 be the
 estimate for $q_{k,i}$. From these estimates, we can generate estimates $\tilde{P}_X = [\estimateProb_1(t), \ldots, \estimateProb_n(t)]^\intercal$ by solving the system of equations described by \Cref{rankEqn}. More precisely, 
 with
$\hat{Q} = [\hat{q}_{1,1}(t), \ldots, \hat{q}_{1,m_1}(t), \ldots \hat{q}_{K,m_K}(t)]^\intercal$, we can solve
\begin{equation}
A\tilde{P}_X = \hat{Q}.
\label{eq:new_linear}
\end{equation}

First, we show that the estimates $\tilde{P}_X(t)$  are unbiased. Let  $Q$ be the list of true probabilities of observations, i.e.,
$Q= [{q}_{1,1}, \ldots, {q}_{1,m_1},{q}_{2,1}, \ldots, {q}_{2,m_2} \ldots {q}_{K,m_K}]^\intercal$. Observe that the length of $Q$ is $m_1 + m_2 + \dots + m_K$. 
The solution of \Cref{eq:new_linear} is given by $\tilde{P}_X = A^+\hat{Q}$, where $A^+$ is the pseudoinverse or the Moore-Penrose inverse of the matrix $A$. Thus, we get
\begin{align}\nonumber 
    \E{\tilde{P}_X} = \E{A^{+}\hat{Q}} = A^{+}\E{\hat{Q}}=A^{+}Q,
\end{align}
upon using the fact that the estimates $\hat{q}_{k,i}(t) = \frac{t_{k,i}}{{t}/{K}}$ are unbiased. Here $t_{k,i}$ denotes the number of times $i^{\text{th}}$ output of arm $k$, i.e $z_{k,i}$, is observed. The desired result
$\E{\tilde{P}_X}=P_X$
is now established as we note that $A^{+}Q = P_X$ in view of \Cref{rankEqn}.
 
Next, we derive a bound on the estimation error $\varepsilon(t)$. It is easy to see 
that the variance of each empirical estimator $\hat{q}_{k,i}(t) = \frac{t_{k,i}}{{t}/{K}}$ is $\OO\left(\frac{K}{t}\right)$. 
With $m=m_1+\ldots+m_K$ denoting the number of rows in $A$, we then get 
\begin{align}
\varepsilon(t)
&= \sum_{j = 1}^{n}\V{\estimateProb_j(t)}
\\
&= \sum_{j =1}^{n} \V{\sum_{h =1}^{m} A^+(j,h)\hat{Q}(h)} 
\\
&\leq  \sum_{j = 1}^{n} \sum_{h = 1}^{m} \left(A^+(j,h)\right)^2 \V{\hat{Q}(h)}
\end{align}
\begin{align}
&= \sum_{j = 1}^{n} \sum_{k = 1}^{K} \sum_{i = 1}^{m_k} \left(A^+(j,s+i)\right)^2 \frac{q_{k,i}(1 - q_{k,i})}{t_k} \label{errorBound} \\
&= \sum_{j = 1}^{n}  \sum_{k=1}^{K}\sum_{i = 1}^{m_k} \left(A^+(j,s+i)\right)^2 \frac{K q_{k,i}(1 - q_{k,i})}{t} \\
&= \OO\left(\frac{1}{t}\right) 
\end{align}
where $s = \sum_{\ell = 1}^{k-1}m_\ell$. The inequality follows from the fact that elements in $\hat{Q}$ are negatively correlated since $\sum_{i=1}^{m_k}{\hat{q}_{k,i}} = 1$ for each $k=1,\ldots,K$.
\end{proof}

\section{Proposed Sequential Distribution Learning Algorithms}
\label{sec:algo}

\begin{figure}[t]
    \centering
    \includegraphics[width=0.45\textwidth]{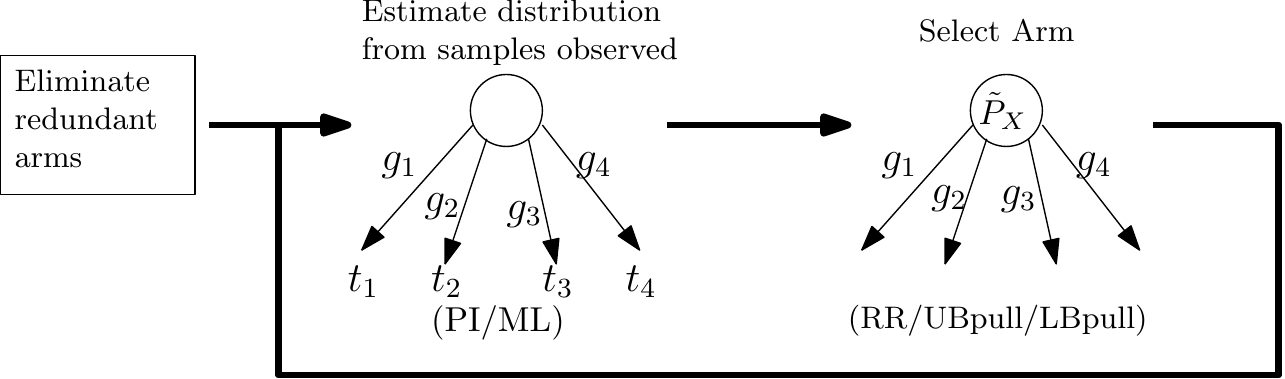}
    \caption{\sl The design of algorithm has two key components: i) Estimating $P_X$ from the samples observed which can be done by the maximum likelihood (ML) or pseudoinverse (PI) estimation schemes ii) Choosing the next arm which can be done in a Round-Robin(RR) manner or by using the \textsc{UBpull}, \textsc{LBpull} strategies.}
    \label{fig:summary1}
\end{figure}

The design of an algorithm to minimize the estimation error can be divided into two parts: 1) producing the estimate of the distribution $\tilde{P}_X(t)$ based on the samples observed till step $t$, and 2) deciding which arm to pull at each time $t$. In \Cref{sec:combining} and \Cref{sec:pulling}, we describe these two parts. \Cref{alg:formalAlgo} and \Cref{alg:formalAlgoLB} describes our proposed algorithms.

\subsection{Combining observations to estimate $P_X$}
\label{sec:combining}

We present a method to estimate $P_X$ given $\mathbf{t} = [t_1, \ldots, t_K]^\intercal$, where $t_k$ is the number of times arm $g_k$ is pulled until time $t$.

In the \textsc{RRpull$+$PIest} Algorithm, the estimate of $P_X$ was obtained using the Moore-Penrose inverse and the empirical probabilities of the observed output, $\hat{Q}$. A drawback of this estimation scheme is that it does not account for the number of times each arm is pulled. Motivated by this we propose the use of Maximum Likelihood Estimator for estimating $P_X$, which takes into account the number of times each arm is pulled to produce estimated probabilities. 

Recall that we defined $t_{k,i}$ as the number of times $i^{th}$ output from arm $k$, i.e., $z_{k,i}$, is observed. Let $\tilde{q}_{k,i}(t)$ be the probability of observing output $z_{k,i}$ under the probability distribution $\tilde{P}_X(t) = [\estimateProb_1(t), \estimateProb_2(t), \ldots, \estimateProb_n(t)]^\intercal$. 
The log likelihood of $\mathcal{D}_t$ with respect to the probability distribution $\estimateProb(t)$ is given by
\begin{equation}
L(\mathcal{D}_t;\tilde{P}_X(t)) = \sum_{k=1}^{K}\sum_{i=1}^{m_k} (t_{k,i}+1) \log(\tilde{q}_{k,i}(t)).
\label{eqn:LogLikelihood}
\end{equation}
where, 
$\tilde{q}_{k,i} = \sum_{j = 1}^{n} A_k(i,j)\estimateProb_j.$ Note that we smooth the log-likelihood by using $t_{k,i}+1$ instead of $t_{k,i}$.
In order to obtain the maximum likelihood estimate of $\estimateProb(t)$, we take the derivative of $L(\mathcal{D}_t;\estimateProb(t))$ and equate it to zero under the constraint $\sum_{i = 1}^{n}\estimateProb_i(t) = 1$. This provides us a set of equations described by 
\begin{equation}
\estimateProb_j(t) = \frac{1}{t} \sum_{k=1}^{K}\sum_{i=1}^{m_k} (t_{k,i}+1)\frac{A_k(i,j)\estimateProb_j(t)}{\tilde{q}_{k,i}(t)}, ~~j =1,2, \ldots, n.
\end{equation}

Observe that these set of equations are in the form of $x = f(x)$ and thus can be solved numerically by finding a fixed point using fixed point iteration method\cite{burden19852}.
Since the log likelihood function is concave in $\estimateProb(t)$, the solution from the set of equations described above maximizes the log likelihood function.
It is known that the Maximum Likelihood Estimate $\hat{\theta}$ of a parameter $\theta$ behaves as $N(\theta, I(\theta)^{-1})$ asymptotically, where $I(\theta)$ is the Fisher Information matrix; here $N(\mu, \sigma^2)$ denotes the normal distribution with mean $\mu$ and variance $\sigma^2$. This means that MLE estimator is asymptotically consistent and belongs to the class of asymptotically normal estimator. Therefore, the lower bound in \Cref{thm:CRLB} holds for MLE estimator and it achieves the stated lower bound asymptotically.

\subsection{Deciding which arm to pull}

\label{sec:pulling}
\Cref{sec:combining} described the Maximum Likelihood estimation approach to estimate $\tilde{P}_X(t)$ from the observations till time step $t$. In this section, we focus on the strategy to pull arm at step $t+1$ given observations till time step $t$. Although the round-robin arm-pulling strategy used in \Cref{alg:NaiveAlgo} achieves order-wise optimal error (\Cref{thm:orderwise}), it has two key drawbacks. Firstly, it is agnostic to the functions $g_k(X)$, and thus even redundant arms will be pulled $t/K$ times. Secondly, it does not consider the distribution estimate $\tilde{P}_X(t)$ when deciding which arm to pull.
%
We now propose an arm-pulling strategy that addresses these shortcomings. The first part of our algorithm involves removal of redundant arms. In the second part we define two strategies, namely \textsc{UBpull} and \textsc{LBpull} that can be used to choose an arm in each step.

\textbf{Removing redundant arms.} For each pair of arms $r,s$ evaluate $rank(A_r), rank(A_s)$ and $rank(B)$, where $B = [A_r^\intercal, A_s^\intercal]^\intercal$. If $\text{rank}(B) = \text{rank}(A_r) > \text{rank}(A_s)$ remove arm $s$. If $\text{rank}(B) = \text{rank}(A_r) = \text{rank}(A_s)$ remove any one of $r$ or $s$ uniformly at random. This leaves us with a new matrix $\bar{A}$ with reduced number of rows (as some arms are removed). Without loss of generality, from now onwards we assume that matrix $A$ does not contain any redundant arm.  

\textbf{The \textsc{UBpull} strategy to choose the next arm.} If we had an analytic expression for estimation error $\varepsilon(t)$ at each step $t$, we could find the arm that minimizes the estimation error. However, in the absence of an invertible arm, it is hard to obtain an analytic expression of $\varepsilon$, due to which we resort to a heuristic approach. In equation \eqref{errorBound} we see an upper bound on estimation error for \textsc{RRpull$+$PIest} algorithm. An approach towards choosing arm could be to minimize this upper bound on the estimation error. However since true probability distribution $P_X$ is unknown, we can obtain an estimate of this upper bound as 
\begin{align}
    U(\tilde{P}_X(t), \mathbf{t}) = \sum_{j = 1}^{n} \sum_{k = 1}^{K} \sum_{i = 1}^{m_k}&\bigg( \left({A}^+(j,s+i)\right)^2 \times \nonumber \\& \frac{\tilde{q}_{k,i}(t)(1 - \tilde{q}_{k,i}(t))}{t_k}\bigg),
\label{eq:upperbound}
\end{align}
where, $s = \sum_{\ell = 1}^{k-1}m_\ell.$

Following this idea, we propose a \textsc{UBpull} decision scheme which makes use of the observations made till time step $t$ to select arm $c_{t+1}$ at time step $t+1$. The \textsc{UBpull} scheme selects an arm $c_{t+1}$, if pulling $c_{t+1}$ would result in maximum decrease of $U(\tilde{P}_X(t), \mathbf{t})$. More formally, we choose $c_{t+1}$ that maximizes 
\begin{equation}
c_{t+1} = \argmax_k U(\tilde{P}_X(t),\mathbf{t}) - U(\tilde{P}_X(t),\bar{\mathbf{t}}^{(k)}),
\label{eqn:V_tilde_def}
\end{equation}
with ties broken uniformly at random. Here $\bar{\mathbf{t}}^{(k)} = \mathbf{t} + \mathbf{e}^{(k)}$, with $\mathbf{e}^{(k)}$ representing a $K$ length column vector with $e^{(k)}_i = 0$ $\forall{i \neq k}$ and $e^{(k)}_k = 1$. This results in $\mathbf{t}^{(k)} = [t_1, t_2, \ldots ,t_k + 1, \ldots t_K]^{\intercal}$. 




The \textsc{UBpull} decision scheme along with the maximum likelihood estimation scheme proposed in \Cref{sec:combining} completes the design of \textsc{UBpull$+$MLest} algorithm.
A formal description of \textsc{UBpull$+$MLest} is presented in \Cref{alg:formalAlgo}.

\begin{algorithm}[t]
\hrule 
\vspace{0.1in}
\begin{algorithmic}[1]
\STATE \textbf{Input:} $\{x_1, x_2, \ldots, x_n\}$, Functions $\{g_1, g_2 \ldots g_K\}$ where $g_i: \{x_1, x_2, \ldots, x_n \} \rightarrow \mathbb{R}$. Total number of steps, $T$. 
\STATE \textbf{Initialize:} $t_{k,i} = 0, \forall{i,k}.$ $\estimateProb_j(0) = \frac{1}{n}, \forall{j}$. 
\STATE Eliminate Redundant Arms
\FOR{$t = 1:T$}
\STATE $c_t = \argmax_k U(\tilde{P}_X(t),\mathbf{t}) - U(\tilde{P}_X(t),\bar{\mathbf{t}}^{(k)})$ 
\STATE Pull arm $c_t$, observe output $y_t$
\IF {$y_t = z_{k,i}$}
\STATE $t_{k,i} = t_{k,i} + 1$
\ENDIF
\STATE Obtain estimates $\estimateProb_j(t)$ by obtaining fixed point solution of the set of equations described by $$\estimateProb_j(t) = \frac{1}{t} \sum_{k=1}^{K}\sum_{i=1}^{m_k} (t_{k,i} + 1)\frac{A_k(i,j)\estimateProb_j(t)}{\tilde{q}_{k,i}(t)},~~$$ $~~$for $j=1,2, \ldots, n.$ 
\ENDFOR
\end{algorithmic}
\vspace{0.1in}
\hrule
\caption{\textsc{UBpull $+$ MLest}}
\label{alg:formalAlgo}
\end{algorithm}

\begin{thm}
\Cref{alg:formalAlgo} does asymptotically consistent estimation whenever $rank(A) = n$.
\label{thm:consistentAlgo}
\end{thm}

\begin{proof}

In order to show \Cref{thm:consistentAlgo}, we first show that under \Cref{alg:formalAlgo} each non-redundant arm is pulled infinitely many times as $t \rightarrow \infty$. More formally, for each non-redundant arm $k$, $t_k \rightarrow \infty$ as $t \rightarrow \infty$. 


Observe that the next arm is selected as 
$$c_{t+1} = \arg\max_k \zeta_k \left(\frac{1}{t_k} - \frac{1}{t_k+1}\right),$$
with, $$\zeta_k = \sum_{i = 1}^{m_k} \left(\sum_{j = 1}^{n} (A^+(j,i+s))^2\right) \tilde{q}_{k,i}(t)(1 - \tilde{q}_{k,i}(t)),$$ where $s = \sum_{\ell = 1}^{k-1}m_\ell.$
We have that $\sum_{i = 1}^{n} (\tilde{A}^+(j,i))^2 > 0$ $\forall{i}$ and from \Cref{lem:qpositive} (See Appendix), $0 < \tilde{q}_{k,i} < 1$ $\forall{k,i}$. This results in $\zeta_k > 0$ $\forall{k}$. 

Let us assume that a non-redundant arm $k$ is pulled only $t_k = \OO(1)$ times in a total of $t$ pulls, where $t \rightarrow \infty$. Any other arm $s$ can only be pulled if $\zeta_s \left( \frac{1}{t_s} - \frac{1}{t_s+1} \right) > \zeta_k \left( \frac{1}{t_k} - \frac{1}{t_k+1} \right)$. Due to this $t_s < \sqrt{\frac{\zeta_s}{\zeta_k}}(t_k + 1)$ $\forall{s \neq k}$. Therefore $t \leq \sum_{s = 1}^{K} \sqrt{\frac{\zeta_s}{\zeta_k}}(t_k + 1) = \OO(1),$ as we assumed $t_k = \OO(1)$. However, this contradicts our assumption that $t \rightarrow \infty$. Therefore as $t \rightarrow \infty$, $t_k \rightarrow \infty$. 

As each non-redundant arm is pulled infinitely many times, we see that each element in the Fisher information matrix \eqref{eq:fisher_info_matrix} approaches infinity. Since the variance in the maximum likelihood estimator  approaches the inverse Fisher information asymptotically, maximum likelihood estimator will asymptotically achieve the bound in \Cref{thm:CRLB}. Since each element in Fisher information approaches infinity, the bound in \Cref{thm:CRLB} approaches zero and consequently $\epsilon(t) \rightarrow 0$ as $t \rightarrow \infty$ for the ML estimator. Therefore the statement in \Cref{thm:consistentAlgo} holds true.
\end{proof}


\textbf{The \textsc{LBpull} strategy to choose next arm.} In the \textsc{UBpull + MLest} algorithm, we used \eqref{eq:upperbound} as a metric for choosing arm at each step. An alternative metric could be the lower bound on estimation error stated in \Cref{thm:CRLB}. Since true probability distribution $P_X$ is unknown, we estimate the expression of \Cref{thm:CRLB} as $$B(\tilde{P}_X(t), \mathbf{t}) = tr(\tilde{I}(\theta, \mathbf{t})^{-1}) + \sum_{i=1}^{n-1} \sum_{j=1}^{n-1} \tilde{I}(\theta, \mathbf{t})^{-1}(i,j),$$
where 
\begin{align}\nonumber
\tilde{I}_{i,j}(\theta,\mathbf{t})  &= \sum_{k=1}^{K}\sum_{\ell = 1}^{m_K} \frac{t_k   A_k(\ell,i)  A_k(\ell,j)  (1 - A_k(\ell,n))}{\tilde{q}_{k,\ell}(t)} + \\
&~ \frac{t_k  (1-A_k(\ell,i))  (1 - A_k(\ell,j)  A_k(\ell,n)}{\tilde{q}_{k,\ell}(t)}.
\label{eq:Pseudo_fisher_info_matrix}
\end{align}

Based on this idea, we propose a \textsc{LBpull} decision scheme which chooses arm $c_{t+1}$ at round $t+1$ if pulling $c_{t+1}$ maximizes the decrease in $B(\tilde{P}_X(t), \mathbf{t})$. More formally, \textsc{LBpull} chooses $c_{t+1}$ that maximizes 
\begin{equation}
c_{t+1} = \argmax_k B(\tilde{P}_X(t),\mathbf{t}) - B(\tilde{P}_X(t),\bar{\mathbf{t}}^{(k)}),
\label{eqn:V_tilde_def}
\end{equation}
with ties broken uniformly at random. As defined earlier, $\bar{\mathbf{t}}^{(k)} = \mathbf{t} + \mathbf{e}^{(k)}$.

This \textsc{LBpull} decision scheme combined with the maximum likelihood estimation scheme completes the design of \textsc{LBpull$+$ MLest} algorithm. A Formal description is presented in \Cref{alg:formalAlgoLB}. We conjecture that the \textsc{LBpull+MLest} scheme also achieves asymptotically consistent estimation whenever possible, we leave the proof as a future work.


\begin{algorithm}[t]
\hrule 
\vspace{0.1in}
\begin{algorithmic}[1]
\STATE \textbf{Input:} $\{x_1, x_2, \ldots, x_n\}$, Functions $\{g_1, g_2 \ldots g_K\}$ where $g_i: \{x_1, x_2, \ldots, x_n \} \rightarrow \mathbb{R}$. Total number of steps, $T$. 
\STATE \textbf{Initialize:} $t_{k,i} = 0, \forall{i,k}.$ $\estimateProb_j(0) = \frac{1}{n}, \forall{j}$. 
\STATE Eliminate Redundant Arms
\FOR{$t = 1:T$}
\STATE $c_t = \argmax_{k} B(\tilde{P}_X(t-1),\mathbf{t}) - B(\tilde{P}_X(t-1),\bar{\mathbf{t}}^{(k)})$ 
\STATE Pull arm $c_t$, observe output $y_t$
\IF {$y_t = z_{k,i}$}
\STATE $t_{k,i} = t_{k,i} + 1$
\ENDIF
\STATE Obtain estimates $\estimateProb_j(t)$ by obtaining fixed point solution of the set of equations described by $$\estimateProb_j(t) = \frac{1}{t} \sum_{k=1}^{K}\sum_{i=1}^{m_k} (t_{k,i}+1)\frac{A_k(i,j)\estimateProb_j(t)}{\tilde{q}_{k,i}(t)},~~$$ $~~$for $j=1,2, \ldots, n.$ 
\ENDFOR
\end{algorithmic}
\vspace{0.1in}
\hrule
\caption{\textsc{LBpull $+$ MLest}}
\label{alg:formalAlgoLB}
\end{algorithm}

\section{Simulation Results}
\label{sec:simulations}

In this section, we demonstrate the performance of our algorithm under different scenarios. We compare the estimation error of our algorithm with the Cram\'er-Rao lower bound evaluated in \Cref{sec:bounds}. Recall that Cram\'er-Rao bound gives a lower bound on the estimation error given the choice of $\{t_1, t_2, \ldots, t_K\}$. To evaluate the lower bound after a total of $t$ time slots, we find the Cram\'er-Rao bound for all combinations of $\{\alpha_1 t, \alpha_2 t, \ldots, \alpha_K t\}$ where $\sum_{i = 1}^{K}\alpha_i = 1$, and take the minimum over all such combinations. 
We iterate $\{\alpha_1, \alpha_2 , \ldots, \alpha_K\}$ for all possible values between 0 to 1 with a precision of 0.001. Note that the existence of an algorithm that achieves the Cram\'er-Rao lower bound is not guaranteed. 

\Cref{fig:comparison} shows the results of our experiment for the example considered in \Cref{fig:cases}. The experiment was repeated 1000 times and we report the average estimation error in the plot. In \Cref{tab:pullsNeededCase1} we report the average number of pulls needed by each algorithm to achieve an error of $10^{-3}$. For comparison purposes we included \textsc{RR}pull+\textsc{ML}est algorithm, which pulls arms in a round-robin manner and produces estimate using maximum likelihood estimation. As evident, the proposed \textsc{UBpull$+$MLest} and \textsc{LBpull$+$MLest} algorithms outperform the \textsc{RRpull$+$PIest} and \textsc{RRpull$+$MLest} algorithms in this scenario. While \textsc{RRpull$+$PIest} algorithm pulls each of the arms equal number of times, the proposed algorithms adapt according to shape of function and the probability distribution estimates to pull each arm different number of times. This is one of the key reason behind the successful performance of \textsc{UBpull$+$MLest} and \textsc{LBpull$+$MLest}. This effect is illustrated in \Cref{fig:pulls_compl}, where we report the average number of times each arm was pulled over 1000 experiments. The combination $\{\alpha_1 t, \alpha_2 t, \ldots, \alpha_K t\}$ resulting in the minimum Cram\'er-Rao bound for $t = 1000$
is displayed in Figures \Cref{fig:comparison} and \Cref{fig:cases} as the \lq\lq CRLB config." We see that the number of times each arm is pulled in \textsc{LBpull$+$MLest} algorithm is very close to these numbers. Given that maximum likelihood estimator is known to achieve Cram\'er-Rao bound asymptotically, this suggests that the asymptotic performance of the \textsc{LBpull$+$MLest} algorithm will be close to optimal.

We now demonstrate why an {\em active} learning framework, where the samples are obtained sequentially {\em based on the current estimate} of $P_X$, is necessary in order to achieve the best performance (e.g., to minimize the error) for the problem under consideration. This is primarily because of the fact that given a total number of available pulls, the optimal number of times that each arm needs to be pulled (in order to minimize the error) depends not only on the functions themselves (or, the sample generation matrix $A$), but also the probability distribution that the algorithm is trying to estimate. It is for this reason that we need an active learning approach where the current estimate of $P_X$ (based on prior samples) is factored into deciding which of the available functions the next sample should come from.

In order to demonstrate the need for an active learning framework, we revisit 
the case considered in \Cref{fig:cases} with the functions $g_1(X), g_2(X)$ and $g_3(X)$ kept the same. This time, we assume that the underlying probability distribution is changed from $P_X = [0.05, 0.1, 0.1, 0.2, 0.2, 0.25, 0.1]$ to  $P_X = [0.4, 0.25, 0.2, 0.05, 0.025, 0.025, 0.05]$. In the former case, we had seen that the lowest error is achieved when functions $g_1, g_2, g_3$ are sampled at a relative fraction of $0.104$, $0.317$, and $0.579$, respectively. In other words, if it is indeed the case that $P_X = [0.05, 0.1, 0.1, 0.2, 0.2, 0.25, 0.1]$, an  algorithm that chooses $g_1, g_2, g_3$ with probabilities $0.104$, $0.317$, and $0.579$, respectively, at each step (independently) would be the optimal in learning this distribution. It might be tempting to think that this {\em baseline} algorithm would do well even if the underlying probability distribution is different, as long as the functions remain the same.
However, under the modified probability distribution  $P_X = [0.4, 0.25, 0.2, 0.05, 0.025, 0.025, 0.05]$, we observe that this baseline algorithm achieves an error which is $22.8\%$ and $31.4\%$ more than \textsc{UBpull$+$MLest} and \textsc{LBpull$+$MLest}, respectively.  This highlights the fact that an algorithm considering only the shape of function may not perform well in all cases and indeed an active learning algorithm that uses the estimates $\hat{p}(t)$ at every step to make the next decision is necessary to tackle this problem (as done by \textsc{UBpull$+$MLest} and \textsc{LBpull$+$MLest}). \Cref{tab:pullsNeeded} illustrates this insight in terms of number of samples required to achieve an error of $10^{-3}$ for \textsc{UBpull$+$MLest}, \textsc{LBpull$+$MLest}, Baseline and \textsc{RRpull$+$PIest} algorithms respectively.

\begin{figure}[t]
    \centering
    \includegraphics[width=0.5\textwidth]{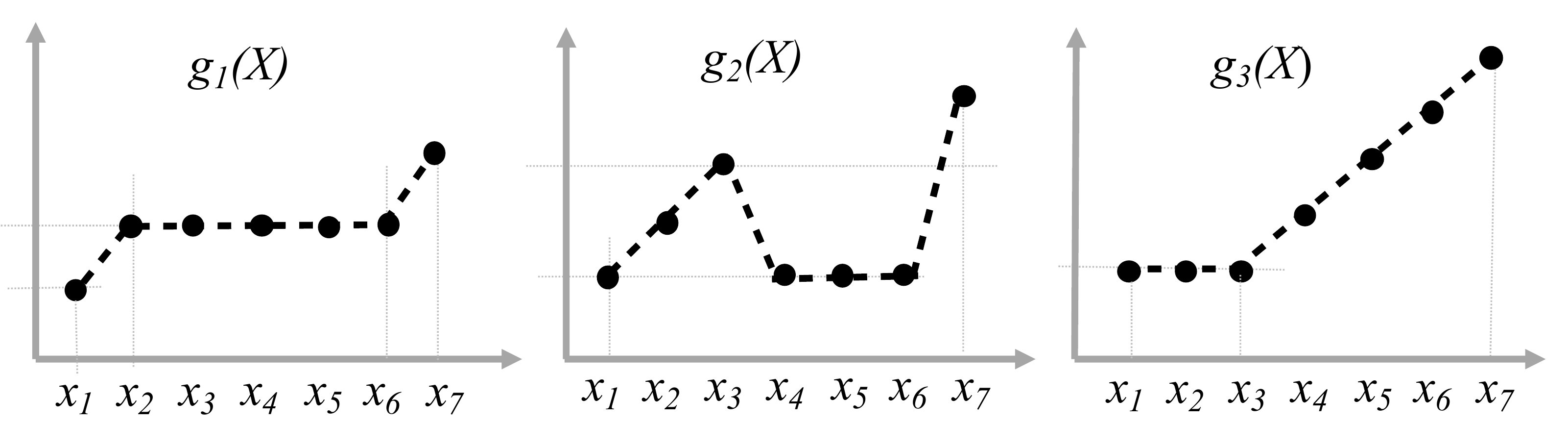}
    \caption{\sl Example of a set of functions $\{g_1, g_2, \ldots, g_K\}$. Probability distribution $P_X = [0.05, 0.1, 0.1, 0.2, 0.2, 0.25, 0.1]$.} 
    \label{fig:cases}
\end{figure}

\begin{figure}[t]
    \centering
    \includegraphics[width=0.5\textwidth]{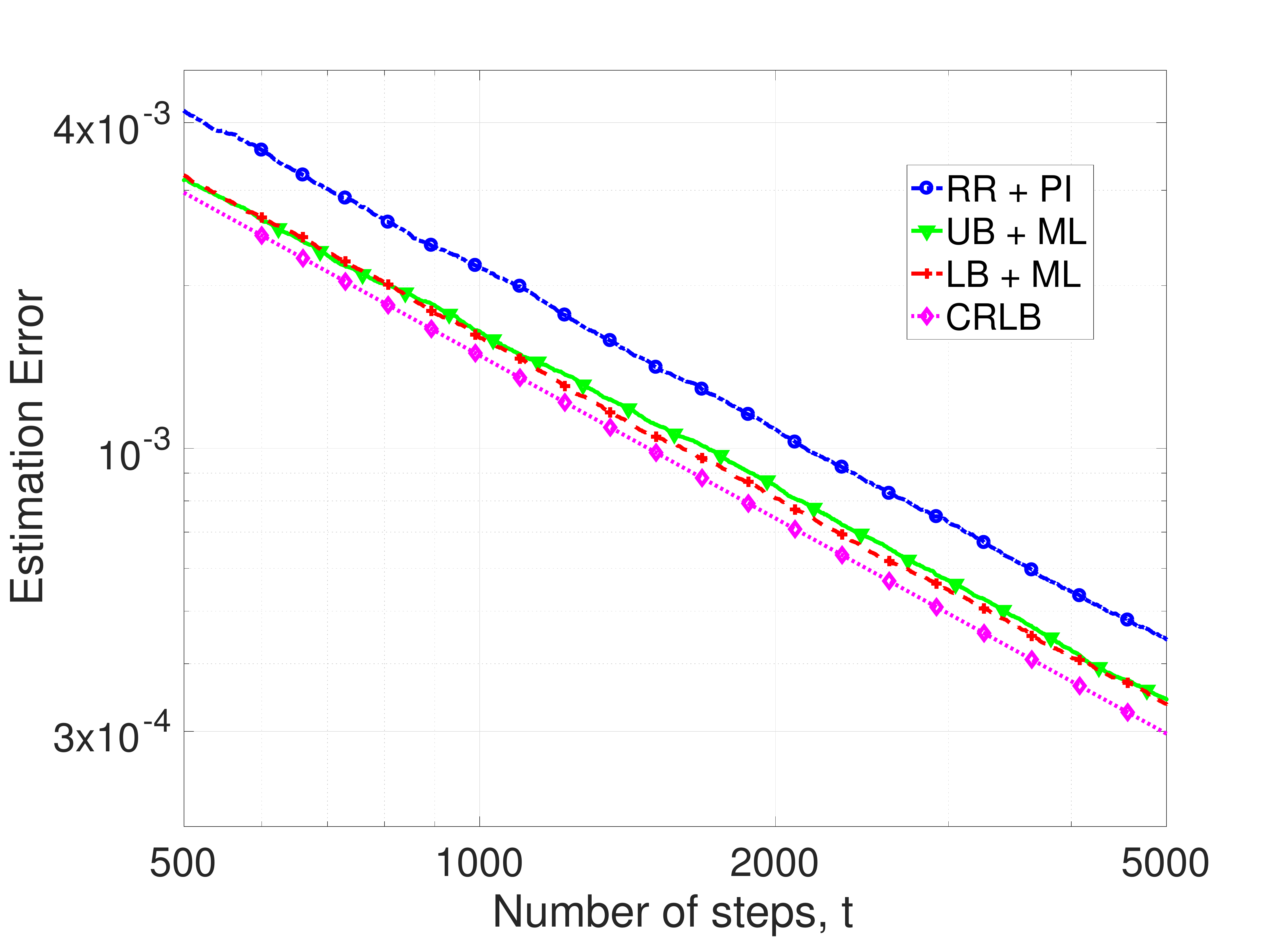}
    \caption{\sl Comparison of our policies against the \textsc{RRpull$+$PIest} algorithm for the example in \Cref{fig:cases} }
    \label{fig:comparison}
\end{figure}

\begin{table}
\centering
 \begin{tabular}{|c c |} 
 \hline
 Algorithm & Avg. pulls needed \\[0.5ex]
 \hline
 \textsc{LBpull$+$MLest} & 2019.1  \\ [0.5ex] 
 \textsc{UBpull$+$MLest} &2126.8 \\[0.5ex]
 \textsc{RRpull$+$MLest} & 2610.7 \\[0.5ex]
 \textsc{RRpull$+$PIest} & 2808.4 \\[0.5ex] 
 \hline
\end{tabular}
\caption{\sl Average number of pulls needed to get an estimation error of $10^{-3}$ for the example in \Cref{fig:cases}.}.
\label{tab:pullsNeededCase1}
\end{table}

\begin{figure}[t]
    \centering
    \includegraphics[width=0.5\textwidth]{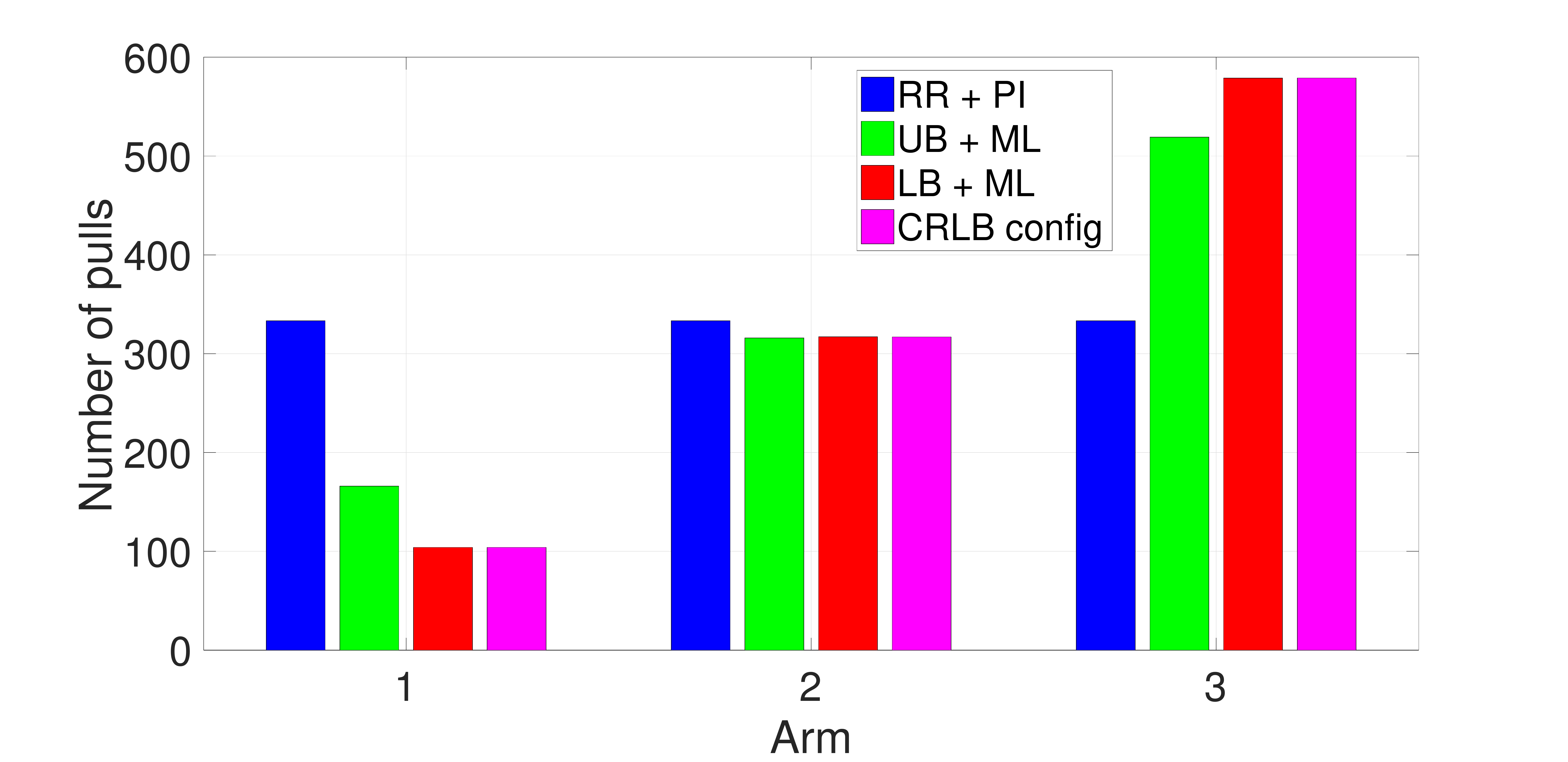}
    \caption{\sl Average number of times each arm is pulled in a total of 1000 steps for the example in \Cref{fig:cases}. Variance in the number of pulls across experiments is small for \textsc{LBpull$+$MLest} and \textsc{UBpull$+$MLest} algorithms.} 
    \label{fig:pulls_compl}
\end{figure}



\begin{table}
\centering
 \begin{tabular}{|c c|} 
 \hline
 Algorithm & Avg. pulls needed \\[0.5ex]
 \hline
 \textsc{LBpull$+$MLest} & 1945.2 \\ [0.5ex]
 \textsc{UBpull$+$MLest} & 2019.2 \\[0.5ex]
 Baseline & 2579.4 \\[0.5ex]
 \textsc{RRpull$+$PIest} & 2347.4 \\ [0.5ex] 
 \hline
\end{tabular}
\caption{\sl Average number of pulls needed to get an estimation error of $10^{-3}$ for the functions in \Cref{fig:cases} where $[p_1, p_2, p_3, p_4, p_5, p_6, p_7]$ is $[0.4, 0.25, 0.2, 0.05, 0.025, 0.025, 0.05]$.}.
\label{tab:pullsNeeded}
\end{table}

\begin{figure}[t]
    \centering
    \includegraphics[width=0.5\textwidth]{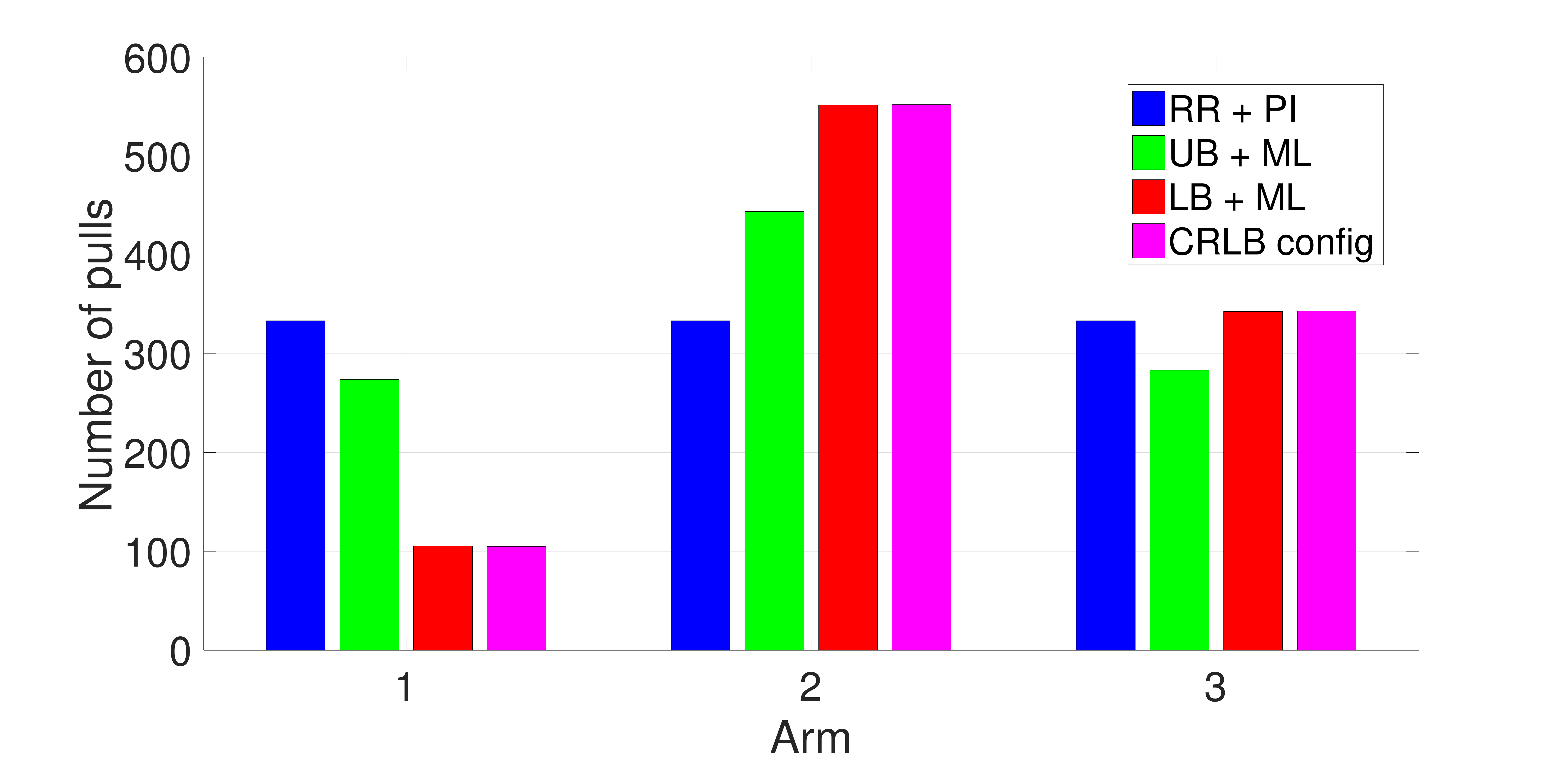}
    \caption{\sl Average number of times each arm is pulled in a total of 1000 steps for the functions in \Cref{fig:cases} where $[p_1, p_2, p_3, p_4, p_5, p_6, p_7]$ is $[0.4, 0.25, 0.2, 0.05, 0.025, 0.025, 0.05]$. Variance in the number of pulls across experiments is small for \textsc{LBpull$+$MLest} and \textsc{UBpull$+$MLest} algorithms.}
    \label{fig:pulls_compl2}
\end{figure}

\section{Concluding Remarks}
\label{sec:conclu}


We consider the problem of learning the distribution $P_X$ of a hidden random variable $X$, using indirect samples from the functions $g_1(X)$, $g_2(X)$, \dots $g_K(X)$, referred to as \emph{arms}. The samples are obtained in a sequential fashion, by choosing one of the $K$ arms in each time slot. Several applications where we wish to infer properties of a hidden random phenomenon using indirect or imprecise observations fit into our framework.
We determine conditions for asymptotically consistent estimation of $P_X$ and evaluate bounds on the estimation error. Using insights from this analysis, we propose algorithms to choose arms and combine their samples. Performance of these algorithms is is shown to outperform several intuitive baseline algorithms numerically.

Ongoing work includes obtaining result on asymptotic consistency for \textsc{LBpull$+$MLest} algorithm. Instead of the deterministic functions $g_i(X)$, we also plan to consider random observations $Y_i$, such that the conditional distribution $p(Y_i|X)$ is known.

\section*{Acknowledgments}
This work was supported in part by the Department of Electrical and Computer Engineering at Carnegie Mellon University and by the National Science Foundation through grants CCF \#1617934 and CCF \#1840860.

\bibliographystyle{ieeetr}
\bibliography{multi_arm_bandit}

\newpage
\appendix


\begin{lem}
The estimates $\tilde{q}_{k,i}(t)$ of output probabilities $q_{k,i}$ produced by \Cref{alg:formalAlgo} are bounded away from $0$ and $1$.
More formally, we have $ \lim \inf_{t \to \infty} \tilde{q}_{k,i}(t) > 0$ and $\lim \sup_{t \to \infty} \tilde{q}_{k,i}(t) < 1$ $\forall{i,k}$.
\label{lem:qpositive}
\end{lem}

\begin{proof}
First, we show that
$0< \tilde{q}_{k,i}(t) < 1$
for all $t=1, 2, \ldots$. 
We recall that the proposed estimation scheme for $\tilde{P}_X(t)$ maximizes the log likelihood in \eqref{eqn:LogLikelihood} at each time step $t$. 
Fix $t=1, 2, \ldots$. 
It is clear from \eqref{eqn:LogLikelihood}
that any estimate $\tilde{P}_X(t)$ for which $\tilde{q}_{k,i} = 0$ for any $k,i$ results in $L(\mathcal{D}_t;\tilde{p}(t)) = -\infty$. On the other hand, any estimate $\tilde{P}_X(t)$ that has $q_{k,i} > 0$ $\forall{k,i}$ (say $\tilde{p}_j(t) = 1/n$  $\forall{j}$) would lead to $L(\mathcal{D}_t, \tilde{P}_X(t)) > -C$, for some $C < \infty$. Since our algorithm returns the estimate
$\tilde{P}_X(t)$ that maximizes 
$L(\mathcal{D}_t, \tilde{P}_X(t))$,
we must have $\tilde{q}_{k,i}(t) > 0$ $\forall{k,i}$. Since $\sum_{i = 1}^{m_k} \tilde{q}_{k,i}(t) = 1$ $\forall{k}$, this in turn implies that $\tilde{q}_{k,i}(t) < 1$ $\forall{k,i}$. This last step also uses the fact that  $m_k > 1$ for all non-redundant arms $k$. 
Combining, we have $0 < \tilde{q}_{k,i}(t) < 1$ for  all $t=1, 2, \ldots$ and 
$\forall{i,k}$.

Next, we show that 
$\lim \inf_{t \to \infty} \tilde{q}_{k,i}(t)>0$. 
Let $\mathcal{S}$ denote the set of arms for which $t_k \rightarrow \infty$. Similarly we denote ${\mathcal{S}^c}$ as the set of arms for which $t_k = \OO(1)$. 
Assume towards a contradiction that the estimates $\tilde{P}_X(t)$
returned by our algorithm satisfies $\lim \inf_{t \to \infty} \tilde{q}_{k,i}(t) = 0$, for some $k,i$. By strong law of large numbers, we have  $t_{k,i} \rightarrow t_k q_{k,i}$ almost surely $\forall{k \in S}$. The log likelihood expression of $\eqref{eqn:LogLikelihood}$ for the aforementioned distribution $\tilde{P}_X(t)$ satisfies
\begin{align}
&L(\mathcal{D}_t;\tilde{P}_X(t)) = \sum_{k = 1}^{K}\sum_{i=1}^{m_k} (t_{k,i}+1) \log(\tilde{q}_{k,i}(t))\\
&= (1+o(1))\sum_{k \in \mathcal{S}}t_k \sum_{i = 1}^{m_k} q_{k,i} \log(\tilde{q}_{k,i}(t)) + \sum_{k \in \mathcal{S}} \sum_{i = 1}^{m_k} \log(\tilde{q}_{k,i}(t)) \nonumber\\
& ~~~~~ + \sum_{k \in \mathcal{S}^c}\sum_{i = 1}^{m_k} (t_{k,i} + 1) \log(\tilde{q}_{k,i}(t)) 
\end{align}
On the other hand, the log likelihood expression for the actual distribution $P_X$ is given by
\begin{align}
&L(\mathcal{D}_t;P_X) = \sum_{k = 1}^{K}\sum_{i=1}^{m_k} (t_{k,i}+1) \log(q_{k,i})\\
&= (1+o(1)) \sum_{k \in \mathcal{S}}t_k \sum_{i = 1}^{m_k} q_{k,i} \log(q_{k,i}) + \sum_{k \in \mathcal{S}} \sum_{i = 1}^{m_k} \log(q_{k,i})  \nonumber\\
& ~~~~~ + \sum_{k \in \mathcal{S}^c}\sum_{i = 1}^{m_k} (t_{k,i} + 1) \log(q_{k,i}) 
\end{align}

Since \Cref{alg:formalAlgo} generates estimates $\tilde{P}_X(t)$ such that for each $t = 1, 2, \ldots$, $\tilde{P}_X(t)$ maximizes $L(\mathcal{D}_t;\tilde{P}_X(t))$ among all possible distributions, we must have 
\begin{equation}
    L(\mathcal{D}_t;P_X) - L(\mathcal{D}_t;\tilde{P}_X(t)) \leq 0 ~ \forall t.
    \label{eqn:likelihoodFact}
\end{equation}  
However, the difference of $L(\mathcal{D}_t;P_X)$ and $L(\mathcal{D}_t, \tilde{P}_X(t))$, with $\tilde{P}_X(t)$ denoting a distribution that satisfies $\lim \inf_{t \to \infty} \tilde{q}_{k,i}(t) = 0$, for some $k,i$, is given by
\begin{align}
&L(\mathcal{D}_t;P_X) - L(\mathcal{D}_t;\tilde{P}_X(t)) 
\nonumber \\
& = (1+o(1)) \sum_{k \in \mathcal{S}}t_k \sum_{i = 1}^{m_k} q_{k,i} \log\left(\frac{q_{k,i}}{\tilde{q}_{k,i}(t)}\right) \nonumber \\
& ~ + \sum_{k \in \mathcal{S}} \sum_{i = 1}^{m_k} \log\left(\frac{q_{k,i}}{\tilde{q}_{k,i}(t)}\right) + \sum_{k \in \mathcal{S}^c}\sum_{i = 1}^{m_k} (t_{k,i} + 1) \log\left(\frac{q_{k,i}}{\tilde{q}_{k,i}(t)}\right) \label{firstOne}\\
&= (1+o(1)) \sum_{k \in \mathcal{S}}t_k \sum_{i = 1}^{m_k} q_{k,i} \log\left(\frac{q_{k,i}}{\tilde{q}_{k,i}(t)}\right) \nonumber \\
&~+ \sum_{k \in \mathcal{S}}\sum_{i = 1}^{m_k} \log (q_{k,i}) + \sum_{k \in \mathcal{S}^c}\sum_{i = 1}^{m_k} (t_{k,i} + 1) \log(q_{k,i}) \nonumber \\
&~+ \sum_{k \in \mathcal{S}}\sum_{i = 1}^{m_k} \log \left(\frac{1}{\tilde{q}_{k,i}(t)}\right) + \sum_{k \in \mathcal{S}^c}\sum_{i = 1}^{m_k} (t_{k,i} + 1) \log \left(\frac{1}{\tilde{q}_{k,i}(t)}\right). 
\end{align}

The first term in \Cref{firstOne} is non-negative since $\sum_{i = 1}^{m_k} q_{k,i} \log\left(\frac{q_{k,i}}{\tilde{q}_{k,i}(t)}\right) \geq 0$, as this is the KL-Divergence between the output probability distribution of arm $k$ and the estimated output probability distribution of arm $k$. Under the assumption that 
 $\lim \inf_{t \to \infty} \tilde{q}_{k,i}(t) = 0$, for some $k,i$, the sum of the second and third terms in \Cref{firstOne} is $\omega(1)$ given that the actual probabilities satisfy $q_{k,i} > 0$ $\forall{k,i}$ (as in our setup $p_j > 0$ $\forall{j}$). Namely, we have
\begin{equation}
    L(\mathcal{D}_t;P_X) - L(\mathcal{D}_t;\tilde{P}_X(t)) = \omega(1).
\end{equation}

This contradicts the fact of \Cref{eqn:likelihoodFact}. 
Thus, the estimates $\tilde{P}_X(t)$ generated by our algorithm can not satisfy $\lim \inf_{t \to \infty} \tilde{q}_{k,i}(t) = 0$, for some $k,i$. In other words, we
have $\lim \inf_{t \to \infty} \tilde{q}_{k,i}(t) > 0$ $\forall{k,i}$. Given that $\sum_{i=1}^{m_k} \tilde{q}_{k,i}(t) = 1$, this also implies $\lim \sup_{t \to \infty} \tilde{q}_{k,i}(t) < 1$ $\forall{i,k}$ establishing Lemma 
\ref{lem:qpositive}.
\end{proof}

\end{document}